\documentclass[preprint,12pt]{elsarticle}

\usepackage{amssymb}
\usepackage{amsmath}
\usepackage{bm}
\usepackage{graphicx}
\usepackage{subfigure}
\usepackage{geometry}
\usepackage{amsthm}
\usepackage{float}
\graphicspath{{pics/}{fig/}}

\newtheorem{Alg}{Algorithm}
\newtheorem{Pro}{Proposition}

\newtheorem{Lem}{Lemma}

\begin{document}

\begin{frontmatter}



\title{Nonlocal Patches based Gaussian Mixture Model for Image Inpainting}


\author{Wei Wan, Jun Liu*}

\address{Laboratory of Mathematics and Complex Systems (Ministry of Education of China), School of Mathematical Sciences, Beijing Normal University, Beijing, 100875, People's Republic of China}
\ead{weiwan@mail.bnu.edu.cn, jliu@bnu.edu.cn}

\begin{abstract}
We consider the inpainting problem for noisy images. It is very challenge
to suppress noise when image inpainting is processed.
An image patches based nonlocal variational method is proposed to simultaneously inpainting and denoising in this paper.
Our approach is developed on an assumption that the small image patches should be obeyed a distribution which can be described by a high dimension Gaussian Mixture Model. By a maximum a posteriori (MAP) estimation, we formulate a new regularization term according to the log-likelihood function of the mixture model. To optimize this regularization term efficiently, we adopt the idea of the Expectation Maximum (EM) algorithm. In which, the expectation step can give an adaptive weighting function which can be regarded as a nonlocal
connections among pixels. Using this fact, we built a framework for non-local image inpainting under noise.
Moreover, we mathematically prove the existence of minimizer for the proposed inpainting model.
By using a spitting algorithm, the proposed model are able to realize image inpainting and denoising simultaneously.
Numerical results show that the proposed method can produce impressive reconstructed results when the inpainting region is rather large.
\end{abstract}


\end{frontmatter}


\section{Introduction}\label{Introduction}

Image inpainting aims to reconstruct the information for the damaged or occluded regions by using the observed information, which is an active topic in image processing and computer vision.
Image degradation is caused by image damage in the process of acquisition, transmission, storage or processing. In addition, image inpainting can also remove certain designated areas in the image according to human purposes, such as text removal and special effects.
Mathematically, image inpainting problem can be described as: suppose that $\Omega$ denotes the image region, $O\subset\Omega$ is the hole or inpainting region, and $O^c:=\Omega\setminus O$ stands for the available information region. $u:\Omega\rightarrow \mathbb{R}$ is the latent clear image, and $v:\Omega\rightarrow \mathbb{R}$ is the observed image. The purpose of image inpainting is to restore the original image $u$ from the observed image $v$.
Actually, it is an interpolation problem from a mathematical viewpoint. But to be different from interpolation, the discontinuity such as image edges and textures preserving is the main difficulty of this problem.

The image inpainting approaches which have been proposed in recent years mainly include: variational and partial differential equation (PDE) based methods and block based methods.
The basic idea of variational and PDE based methods is to extend the structure present in the area surrounding the missing region.
In \cite{Bertalmio2005Image}, Bertalmio \textit{et al.} first proposed an image inpainting model based on third-order PDE (known as BSCB). This approach aims to propagate the information smoothly from the surrounding region into the inpainting region along the direction of the isophotes (i.e. level lines of equal gray values).
In addition, inspired by the classical \textit{total variation }(TV) denoising model \cite{Rudin1992Nonlinear}, Chan and Shen \cite{Chan2001Mathematical} proposed the TV inpainting model. However, it is limited by the size of the inpainting region and does not satisfy the Connectivity Principle. Furthermore, the \textit{Curvature-Driven Diffusions} (CDD) inpainting model was developed by Chan and Shen \cite{Chan2001Nontexture}, which overcomes the connectivity problem by introducing the curvature of the isophotes. In \cite{Chan2002Euler}, Chan, Kang and Shen studied a variational inpainting model based on Euler's elastica energy and analyzed the connections to BSCB model and CDD model. In addition, some other typical variational and PDE based methods are the Mumford-Shah-Euler inpainting model \cite{Esedoglu2002Digital}, the inpainting methods based on the Cahn-Hilliard equation \cite{Bertozzi2007Inpainting} \cite{Burger2009Cahn}, and inpainting with the TV-stokes equation \cite{Tai2007Image} and so on. In general, these methods show good performance in geometry images (also known as cartoon or structure images) with small image gaps. However, they can not address the texture inpainting when the missing 
regions in an image are large.

Image patches based methods try to fill in the missing region by copying blocks from the known regions. When searching for a matching block, the whole image can be scanned. Thus block based methods are non-local, which can effectively fill in the texture regions and repetitive structures. In recent years, block based non-local methods have become popular for image restoration \cite{Jun2012A, Hui2013Robust, Katkovnik2010From, Kawai2009Image, Wexler2004Space, Aujol2008Exemplar, Cao2009Geometrically}.
For image inpainting, Demanet \textit{et al.} \cite{Demanet2003Image} proposed that block based inpainting methods can be regarded as finding a correspondence map $F:O\rightarrow O^c$. Thus each pixel value $x \in O$ can be computed by $u(x):=u(F(x))$.
In \cite{Criminisi2004Image}, Criminisi \textit{et al.} developed an exemplar based algorithm by considering inpainting order, i.e. patch priority, in which texture and structure information are propagated simultaneously.
In \cite{Li2014A}, Li \textit{et al.} proposed a universal algorithm framework called iterative decoupled inpainting (IDI) for image inpainting.
Then patch based regularization operator such as BM3D transform can be used in this method.
Meanwhile, a block based inpainting model using group sparsity and TV regularization was considered by Wan \textit{et al.} \cite{weiwan2018}. The updated local SVD operators are effective in promoting the sparse representation and play the role of dictionary learning.
In \cite{Arias2011A}, a variational framework for block based image inpainting was proposed by Arias \textit{et al.}. Their proposed energy functional is
\begin{align}\label{eq:E1}
{
E_1(u,w)= \int_{\widetilde{O}} \int_{\widetilde{O}^c}w(x,y)\varepsilon(u_B(x)-u_B(y))dydx + h\int_{\widetilde{O}} \int_{\widetilde{O}^c} w(x,y)\ln w(x,y)dydx},
\end{align}
where $w:\widetilde{O}\times \widetilde{O}^c\rightarrow \mathbb{R}$ is a similarity weight function that measures the similarity between patches centered in $\widetilde{O}$ and $\widetilde{O}^c$. The patch error function $\varepsilon$ is defined as
$$
\varepsilon(u_B(x)-u_B(y)):=\int_{B_r}g(z)[u(x+z)-u(y+z)]^2dz,
$$
when the weighted squared $L^2$-norm (i.e. the patch non-local means scheme) is taken and $g$ is a Gaussian function.
The negative of the second term represents the entropy regularization, however, the authors did not give it a mathematical interpretation. In \cite{Liu2017A}, Liu \textit{et al.} proposed a block based non-local maximum a posterior estimation (MAP) framework for image denoising. By utilizing the well-known expectation maximum (EM) algorithm, they developed a block diffusion based BNL$H^1$ model:
\begin{align*}
E_2(u,w)=\frac{\lambda}{2}\int_{\Omega}(u(x)-v(x))^2dx+h\int_{\Omega}\int_{\Omega}w(x,y)\ln w(x,y)dydx+(H^1)_w^g(u),
\end{align*}
where
$$
(H^1)_w^g(u)=\int_{\Omega}\int_{\Omega}\int_{B_r}g(z)[u(y+z)-u(x+z)]^2w(x,y)dzdydx.
$$
The first term is the fidelity term used to measure the difference between the noisy image $v$ and clear image $u$. The second and third terms can be regarded as regularization terms derived from EM algorithm. Here we extend our previous block based denoising work \cite{Liu2017A} to image inpainting problem and present a new statistical interpretation for the entropy regularization.

The majority of image inpainting work assumes the original images are not polluted by noise. In practice, they are usually polluted by noise for various unavoidable reasons.
In this paper, we assume the inpainting image is destroyed with larger missing region and also polluted by Gaussian white noise.
Obviously, this problem is more challenging than traditional image inpainting.
We propose a novel block based non-local inpainting model which is derived from MAP estimation. By the assumption of the self similarity of small patches, we construct a nonparametric mixture model to describe their prior probability density function. In order to optimize this regularization term derived from the mixture model efficiently, we employ the idea of EM algorithm and give a variational framework for non-local image inpainting. We mathematically prove the existence of minimizer for the proposed model. Moreover, this proposed model combines image inpainting and denoising process into a unified framework. It is a  generalization of the traditional block based image inpainting model \eqref{eq:E1}. By using multiscale technique appeared in \cite{Arias2011A}, the proposed model are able to produce better results when the inpainting region is rather large. We propose a decoupling algorithm and experimental results show that it can lead to some impressive image restoration results.

The rest of the  paper is organized as follows. In Section \ref{TheProposedMethod}, we propose the block based non-local inpainting model. The existence of minimizer for the proposed model is proved in Section \ref{ExistenceofMinimizer}. In Section \ref{Algorithms}, we present the algorithm and some details about the implementation. Experimental results are performed in Section 5. Finally, we conclude this work in Section 6.

\section{The Proposed Method}\label{TheProposedMethod}
\subsection{Block Based Maximum a Posteriori (MAP) Estimation}
We assume that the original image $u$ is polluted by Gaussian noise, i.e. $v(x)=u(x)+\eta(x)$, $\eta(x)\sim N(0,{\sigma}^2)$, $x \in O^c$. Denote $B_r$ as a small symmetrical neighborhood centered at $(0, 0)$, i.e. $B_r=\{x:\|x\|\leq r\}$. Thus, a small image block of $u$ centered at $x$ is denoted by $u_B(x)=\{u(x+z):z \in B_r\}$. Furthermore, we take the extended inpainting region $\widetilde{O}$ as the set of centers of blocks that intersect the inpainting region $O$, i.e. $\widetilde{O}:=O+B_r=\{x \in \Omega:(x+B_r)\bigcap O\neq \emptyset\}$. This will help in transferring available information, which is equivalent to the boundary condition of PDE.
\begin{figure}[h]
  \begin{center}
  \includegraphics[width=.5\linewidth]{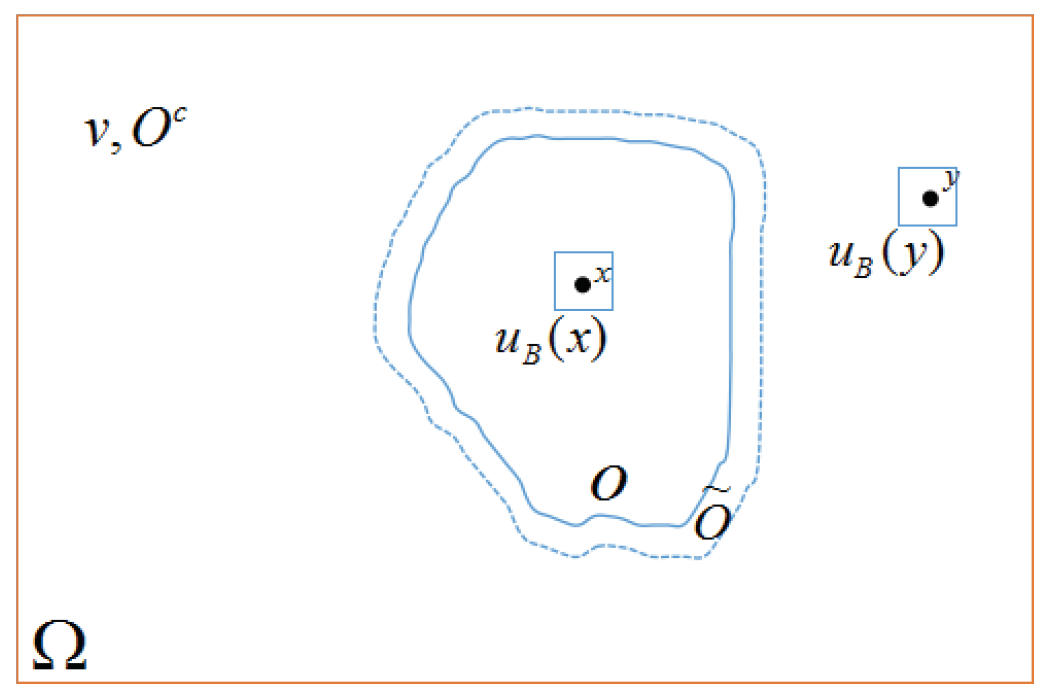}
  \end{center}
  \caption{Inpainting problem.}
\label{figu1}
\end{figure}

According to MAP estimation problem, we can get
\begin{align*}
u^*=\arg\max\limits_{u} p(u_B|v_B)= \arg\max \limits_{u}\frac{p(v_B|u_B)p(u_B)}{p(v_B)}.
\end{align*}
Once the observed image $v$ is given, $p(v_B)$ will be known. Therefore, the MAP problem is equivalent to
\begin{align}\label{2.1}
u^*=\arg\min\limits_{u}\{-\ln p(v_B|u_B)-\ln p(u_B)\}.
\end{align}

Since $v(x)=u(x)+\eta(x),\eta(x)\sim N(0,{\sigma}^2),x \in O^c$, one can easily find
$$
p(v_B(x)|u_B(x))\propto \exp\large\left(-\frac{(u(x)-v(x))^2}{2\sigma^2}\large\right),
$$
where $\sigma$ is the standard deviation of the Gaussian noise. Then according to the independent identify distribution assumption for $\eta(x),x \in \widetilde{O}^c$, we have
$$
p(v_B|u_B)=\prod_{x \in \widetilde{O}^c}p(v_B(x)|u_B(x))\propto\prod_{x \in \widetilde{O}^c}\exp\large\left(-\frac{(u(x)-v(x))^2}{2\sigma^2}\large\right).
$$

Next we apply block based approach to compute a priori term of $u$. Once $u$ is estimated, the histogram of all the block $u_B(y),y \in \widetilde{O}^c$ can be calculated by
$$
\sum_{y \in \widetilde{O}^c}\delta(\mathbf{z}-u_B(y)).
$$
Where $\delta$ is a vector-valued function which counts the number of blocks whose structures are the same as $\mathbf{z}$, i.e.
\begin{align*}
\delta(\mathbf{x})=\begin{cases}
1,&\mathbf{x}=0\\
0,&else
\end{cases}.
\end{align*}

For calculation convenience, usually $\delta$ can be approximated by a Gaussian function
$$
\delta_h(\mathbf{x})=\frac{1}{(\pi h)^\frac{|B|}{2}}\exp \large\left(-\frac{\|\mathbf{x}\|^2}{h}\large\right),
$$
where $h>0$ is a parameter which controls the precision of this approximation.

Assume that all the clear blocks $u_B(x),x \in \widetilde{O}^c$ are some realizations of the random vector $\mathbf{z} \in \mathbb{R}^d (d=2r\times 2r)$ with the probability density function
$$
p(\mathbf{z})=\frac{1}{|\widetilde{O}^c|}\sum_{y \in \widetilde{O}^c}\delta_h(\mathbf{z}-u_B(y)),
$$

Then, from the independent identify distribution assumption, we have
$$
p(u_B)=\prod_{x \in \Omega}p(u_B(x))=\frac{1}{|\widetilde{O}^c|^{|\Omega|}}\prod_{x \in \Omega}\sum_{y \in \widetilde{O}^c}\delta_h(u_B(x)-u_B(y)).
$$

Therefore, ignoring any constant terms, the problem \eqref{2.1} becomes
\begin{align}\label{2.2}
u^*=\arg\min\limits_{u}\biggl\{E(u)=\frac{1}{2\sigma^2}\sum_{x \in \widetilde{O}^c}(u(x)-v(x))^2-\sum_{x \in \Omega}\ln\sum_{y \in \widetilde{O}^c}\delta_h(u_B(x)-u_B(y))\biggl\}.
\end{align}
\subsection{Expectation Maximum (EM) Process and Weighting Function}
Since the second regularization term contains the sum in the log function, which is a log-likelihood function for a Gaussian mixture model, it is not easy to optimize $E(u)$ directly. For convenience, we write
$$
E_1(u)=-\sum_{x \in \Omega}\ln\sum_{y \in \widetilde{O}^c}\delta_h(u_B(x)-u_B(y)).
$$

We observe that $E_1(u)$ is a standard parameters estimation problem of mixture model, thus we can use the well-known expectation maximum (EM) algorithm to solve this problem efficiently. We assume that all the small image blocks $u_B(x),x \in \widetilde{O}^c$ can be divided into $|\widetilde{O}^c|$ class. Now we introduce a latent integer random variable $\mathbf{C}=\{\mathcal{C}_1,\mathcal{C}_2,\cdots,\mathcal{C}_{|\widetilde{O}^c|}\}$ as the classification. Then, we use a function $w(x,y)$ to represent the probability of each block $u_B(x),x \in \Omega$ belongs to $y$-th class. According to the standard EM method, the expectation step can be written as (see Appendix)\\
E-Step:
\begin{align}\label{2.3}
w^{n}(x,y)=P(\mathcal{C}_i=y|u_B^n(x))=\frac{\delta_h(u_B^n(x)-u_B^n(y))}{\sum\limits_{y \in \widetilde{O}^c}\delta_h(u_B^n(x)-u_B^n(y))}.
\end{align}
M-Step:
\begin{align}\label{2.4}
u^{n+1}=\arg\min_{u}\biggl\{-\sum_{x \in \Omega}\sum_{y \in \widetilde{O}^c}\ln(\delta_h(u_B^n(x)-u_B^n(y)))w^{n}(x,y)\biggl\}.
\end{align}

Therefore, we will get a minimizer of $E_1(u)$ by alternatively calculating E-step and M-step until the stopping criteria are met. Hence, EM algorithm reduces computational complexity of minimizing $E_1(u)$. In addition, we can also minimize $E_1(u)$ directly using the following lemma, which is equivalent to EM algorithm.\\
\begin{Lem}(Commutativity of log-sum operations) Given a function $f(x,y)>0$, we have
$$
-\sum_{x \in \Omega}\ln\sum_{y \in \widetilde{O}^c}f(x,y)=\min_{w \in \mathbb{W}}\biggl\{-\sum_{x \in \Omega}\sum_{y \in \widetilde{O}^c}\ln f(x,y)w(x,y)+\sum_{x \in \Omega}\sum_{y \in \widetilde{O}^c}w(x,y)\ln w(x,y)\biggl\},
$$
where $\mathbb{W}=\biggl\{w:\Omega\times \widetilde{O}^c\rightarrow \mathbb{R}:0\leq w(x,y) \leq 1,\sum\limits_{y \in \widetilde{O}^c}w(x,y)=1,\forall x \in \Omega\biggl\}$.\\
\end{Lem}
\begin{proof} Set
$$
\mathcal{P}(w)=-\sum_{x \in \Omega}\sum_{y \in \widetilde{O}^c}\ln f(x,y)w(x,y)+\sum_{x \in \Omega}\sum_{y \in \widetilde{O}^c}w(x,y)\ln w(x,y).
$$
The Lagrangian function of the considered optimization problem is
$$
G(w,\lambda)=\mathcal{P}(w)+\sum_{x \in \Omega}\lambda(x)\large\left(\sum_{y \in \widetilde{O}^c}w(x,y)-1\large\right).
$$
According to the necessary condition of seeking extreme value point $\frac{\partial G}{\partial w}=0$, we deduce that
\begin{align}\label{2.5}
w^*=f(x,y)\exp(-1-\lambda(x)).
\end{align}
Then, since $\sum\limits_{y \in \widetilde{O}^c}w(x,y)=1$, we have
\begin{align}\label{2.6}
\lambda(x)+1=\ln\sum_{y \in\widetilde{O}^c}f(x,y).
\end{align}
Using \eqref{2.5} and \eqref{2.6}, we calculate that
\begin{align}\label{2.7}
w^*=\frac{f(x,y)}{\sum\limits_{y \in \widetilde{O}^c}f(x,y)}.
\end{align}
It is easy to check that $w^*$ is the minimizer of $\mathcal{P}(w)$. Therefore taking \eqref{2.7} into $\mathcal{P}$, we conclude
$$
\mathcal{P}(w^*)=-\sum_{x \in \Omega}\ln\sum_{y \in \widetilde{O}^c}f(x,y).
$$
\end{proof}

We define
$$
H(u,w):=-\sum_{x \in \Omega}\sum_{y \in \widetilde{O}^c}\ln \delta_h(u_B(x)-u_B(y))w(x,y)+\sum_{x \in \Omega}\sum_{y \in \widetilde{O}^c}w(x,y)\ln w(x,y),
$$
and according to Lemma 1, we see that
$$
\min_{u}E_1(u)=\min_u\{\min_{w \in \mathbb{W}}H(u,w)\}.
$$

Next, we can solve this problem with the alternating direction minimization method as follows.
\begin{align}\label{2.9}
\begin{cases}
w^{n+1}=\arg\min\limits_{w \in \mathbb{W}}H(u^n,w),\\
u^{n+1}=\arg\min\limits_{u} H(u,w^{n+1}).
\end{cases}
\end{align}
Obviously, the two problem in \eqref{2.9} are equivalent to the E-step \eqref{2.3} and the M-step \eqref{2.4} in the usual EM algorithm, respectively.

Thus, based on the above analysis, we propose the following image inpainting functional
\begin{align*}
(u^*,w^*)&=\arg\min\limits_{u,w \in \mathbb{W}}\biggl\{\frac{1}{2\sigma^2}\int_{\widetilde{O}^c}(u(x)-v(x))^2dx+\int_{\Omega}\int_{\widetilde{O}^c}w(x,y)\ln w(x,y)dydx\\
&-\int_{\Omega}\int_{\widetilde{O}^c}\ln \delta_h(u_B(x)-u_B(y))w(x,y)dydx
\biggl\}.
\end{align*}

Since $\delta_h(\mathbf{x})=\frac{1}{(\pi h)^\frac{|B|}{2}}\exp \large\left(-\frac{\|\mathbf{x}\|^2}{h}\large\right)$, then we obtain
\begin{align*}
(u^*,w^*)&=\arg\min\limits_{u,w \in \mathbb{W}}\biggl\{\frac{\lambda}{2}\int_{\widetilde{O}^c}(u(x)-v(x))^2dx
+h\int_{\Omega}\int_{\widetilde{O}^c}w(x,y)\ln w(x,y)dydx\\
&+\int_{\Omega}\int_{\widetilde{O}^c}\|u_B(x)-u_B(y)\|^2w(x,y)dydx\biggl\}.
\end{align*}
\subsection{Patch Error Function}
Patch error function is defined as
$$
\|u_B(x)-u_B(y)\|^2:=\int_{B_r}g(z)[(\rho_\varepsilon *u)(x+z)-(\rho_\varepsilon *u)(y+z)]^2dz,
$$
where $g:B_r \rightarrow \mathbb{R}^+,\rho_\varepsilon:B_\varepsilon\rightarrow \mathbb{R}^+$ are two intra-block weight function with $g(z)=g(-z)$, $\int_{B_r} g(z)dz=1$ and $\int_{B_\varepsilon} \rho_\varepsilon(z)dz=1$.
The function $g(z)$ is usually chosen as Gaussian function, which can guarantee that the weight of pixel errors closer to the center of the block is larger, while the weight of pixel errors closer to both sides of the block is smaller.
The symbol "$\ast$" stands for the convolution operator, namely, $(\rho_{\varepsilon} * u)(x)=\int_{B_\varepsilon}\rho_{\varepsilon}(y)u(x-y)dy$.


Therefore, the previous model can be modified as the following form
\begin{align*}
(u^*,w^*)&=\arg\min\limits_{u,w \in \mathbb{W}}\biggl\{\frac{\lambda}{2}\int_{\widetilde{O}^c}(u(x)-v(x))^2dx+h\int_{\Omega}\int_{\widetilde{O}^c}w(x,y)\ln w(x,y)dydx\\
&+\int_{\Omega}\int_{\widetilde{O}^c}\int_{B_r}g(z)[(\rho_\epsilon*u)(x+z)-(\rho_\epsilon*u)(y+z)]^2w(x,y)dzdydx
\biggl\},
\end{align*}
where $\mathbb{W}=\biggl\{w:\Omega \times \widetilde{O}^c\rightarrow \mathbb{R}:0\leq w(x,y) \leq 1,\mathrm{and}\int_{\widetilde{O}^c}w(x,y)dy=1,\forall x \in \Omega \biggl\}.$

Notice that $\Omega=\widetilde{O}\bigcup\widetilde{O}^c$, thus
\begin{align*}
(u^*,w^*)&=\arg\min\limits_{u,w \in \mathbb{W}}\biggl\{\frac{\lambda}{2}\int_{\widetilde{O}^c}(u(x)-v(x))^2dx+h\int_{\widetilde{O}^c}\int_{\widetilde{O}^c}w(x,y)\ln w(x,y)dydx\\
&+\int_{\widetilde{O}^c}\int_{\widetilde{O}^c}\int_{B_r}g(z)[(\rho_\epsilon*u)(x+z)-(\rho_\epsilon*u)(y+z)]^2w(x,y)dzdydx\\
&+h\int_{\widetilde{O}}\int_{\widetilde{O}^c}w(x,y)\ln w(x,y)dydx\\
&+\int_{\widetilde{O}}\int_{\widetilde{O}^c}\int_{B_r}g(z)[(\rho_\epsilon*u)(x+z)-(\rho_\epsilon*u)(y+z)]^2w(x,y)dzdydx
\biggl\}\\
&:=\arg\min\limits_{u,w \in \mathbb{W}}\biggl\{F_1(u,w)+F_2(u,w)\biggl\}.
\end{align*}
where $\mathbb{W}=\biggl\{w:\Omega \times \widetilde{O}^c\rightarrow \mathbb{R}:0\leq w(x,y) \leq 1,\mathrm{and}\int_{\widetilde{O}^c}w(x,y)dy=1,\forall x \in \Omega \biggl\}.$

The first three terms $F_1(u,w)$ in the above energy functional are used for denoising on $O^c$, i.e. the BNL$H^1$ denoising model in \cite{Liu2017A}, while the last two terms $F_2(u,w)$ is for inpainting on $O$, i.e. the patch non-local means inpainting scheme in \cite{Arias2011A}. Therefore, in the proposed model, inpainting and denoising process are implemented simultaneously.

\section{Existence of Minimizer}\label{ExistenceofMinimizer}
In this section, we will prove the existence of minimizer for the proposed model.
\begin{align*}
J(u,w)&=\frac{\lambda}{2}\int_{\widetilde{O}^c}(u(x)-v(x))^2dx+h\int_{\Omega}\int_{\widetilde{O}^c}w(x,y)\ln w(x,y)dydx\\
&+\int_{\Omega}\int_{\widetilde{O}^c}\int_{B_r}g(z)[(\rho_\epsilon*u)(x+z)-(\rho_\epsilon*u)(y+z)]^2w(x,y)dzdydx.
\end{align*}

We will show the existence of minimizer in the following space
$$
X:=\{(u,w):u \in L^2(\Omega),w \in \mathbb{W}\subset L^\infty(\Omega,\widetilde{O}^c) \}.
$$
\begin{Pro} There exists a solution of the variational problem
\begin{equation*}
\inf\limits_{(u,w) \in X} J(u,w).\tag{$\star$}
\end{equation*}
\end{Pro}
\begin{proof}
\begin{align*}
J_1(u)&=\frac{\lambda}{2}\int_{\widetilde{O}^c}(u(x)-v(x))^2dx,\\
J_2(w)&=h\int_{\Omega}\int_{\widetilde{O}^c}w(x,y)\ln w(x,y)dydx,\\
J_3(u,w)&=\int_{\Omega}\int_{\widetilde{O}^c}\int_{B_r}g(z)[(\rho_\epsilon*u)(x+z)-(\rho_\epsilon*u)(y+z)]^2w(x,y)dzdydx.
\end{align*}

Observe that $J_1(u)\geq 0$ and $J_3(u,w)\geq 0$. 
Furthermore, since $x\ln x\geq -\frac{1}{e}$ whenever $x \geq 0$ and $\Omega$ is a bounded domain, we have that $J_2(w)\geq -\frac{h}{e}|\widetilde{O}^c||\Omega|$. Thus $J(u,w)$ has a lower bound and $\inf\limits_{(u,w) \in X} J(u,w)$ exists. Let $\{(u_n,w_n)\}$ be a minimizing sequence of ($\star$), i.e. a sequence such that
$$
J(u_n,w_n)\rightarrow \inf\limits_{(u,w) \in X} J(u,w).
$$

It is clear that $\|w_n\|_{L^\infty}\leq 1$. By the Banach-Alaoglu Theorem, we conclude that there exists a * weakly convergent subsequence (which we relabel by n) and a * weak limit $w \in L^\infty (\Omega,\widetilde{O}^c)$ such that
$$
w_n \stackrel{\text{*}}{\rightharpoonup} w~~~ \mathrm{in}~~~ L^\infty (\Omega,\widetilde{O}^c).
$$
Since $\mathbb{W}$ is convex, closed subset of $L^\infty (\Omega,\widetilde{O}^c)$, and so, by Mazur's Theorem, $w \in \mathbb{W}$. From $w \ln w$ is convex with respect to $w$, $J_2(w)$ is * weakly lower semicontinuous, i.e.
$$
\liminf_{n\rightarrow \infty}J_2(w_n)\geq J_2(w).
$$

Since $J(u,w)$ is coercive, ${u_n}$ must be bounded, i.e. there is a constant $M>0$ such that
$$
\|u_n\|_{L^2(\Omega)}\leq M.
$$
Boundedness of the sequence in a reflexive space implies the existence of a weakly convergent subsequence, which we still denote by ${u_n}$, thus there exists a weak limit $u$ such that
$$
u_n\rightharpoonup u~~~ \mathrm{in}~~~ L^2(\widetilde{O}^c).
$$
Since $J_1:u\mapsto \frac{\lambda}{2}\int_{\widetilde{O}^c}(u(x)-v(x))^2dx$ is continuous and convex, it follows that
$$
\liminf_{n\rightarrow \infty}J_1(u_n)\geq J_1(u).
$$

Set $b_n(x,y)=\int_{B_r}g(z)[(\rho_\epsilon*u_n)(x+z)-(\rho_\epsilon*u_n)(y+z)]^2dz$, it is obviously that $b_n(x,y)$ is bounded. Now let us prove $\frac{\partial b_n}{\partial x}(x,y)$ and $\frac{\partial b_n}{\partial y}(x,y)$ are bounded. We write
\begin{align*}
\frac{\partial b_n}{\partial x}(x,y)&=2\int_{B_r}g(z)[(\rho_\epsilon*u_n)(x+z)-(\rho_\epsilon*u_n)(y+z)]\large\left[(\frac{\partial \rho_\epsilon}{\partial x}*u_n)(x+z)\large\right]dz,\\
\frac{\partial b_n}{\partial y}(x,y)&=-2\int_{B_r}g(z)[(\rho_\epsilon*u_n)(x+z)-(\rho_\epsilon*u_n)(y+z)]\large\left[(\frac{\partial \rho_\epsilon}{\partial y}*u_n)(y+z)\large\right]dz.
\end{align*}
As the boundedness of $u_n$, we deduce that ${b_n(x,y)}$ is a bounded sequence in $W^{1,\infty}(\Omega,\widetilde{O}^c)$. By Compactness Embedding Theorem, if necessary, we may assume that ${b_n(x,y)}$ converges in $L^1(\Omega,\widetilde{O}^c)$ to some $b(x,y)$. Note that
$$
w_n \stackrel{\text{*}}{\rightharpoonup} w~~~ \mathrm{in}~~~ L^\infty (\Omega,\widetilde{O}^c).
$$
Thus
$$
\int_\Omega\int_{\widetilde{O}^c} w_{n}(x,y)b_n(x,y)dydx\rightarrow\int_\Omega\int_{\widetilde{O}^c} w(x,y)b(x,y)dydx,~~n\rightarrow +\infty,
$$
i.e.
$$
\lim_{n\rightarrow \infty}J_3(u_n,w_n)= J_3(u,w).
$$
Consequently
$$
J(u,w)=J_1(u)+J_2(w)+J_3(u,w) \leq \liminf_{n\rightarrow \infty}J(u_n,w_n),
$$
and
$$
J(u,w)=\inf\limits_{(u,w) \in X} J(u,w).
$$
\end{proof}

\section{Algorithms}\label{Algorithms}
The symbol "$\sim$" is denoted as an expansion of a function. For example, $\widetilde{w}(x-z,y)$ is an expansion of $w(x-z,y)$ with
$$\widetilde{w}(x-z,y)=\Bigg\{\begin{array}{cc}
                                      w(x-z,y),& (x,y) \in \Omega\times \widetilde{O}^c\\
                                      0,& else
                                    \end{array}
\Bigg..$$
To simplify the computation, we can set $\rho_{\varepsilon}$ to the delta function $\delta$, then $\rho_{\varepsilon}*u=\delta*u=u$. Furthermore, to minimize the energy $J$ directly, we can use an alternate minimization algorithm. This process is summarized in Algorithm 1.
\begin{Alg} Given an initial value $u^0=v$, for $n=1,2,3,\cdots,$ do

Step 1. Update Weights:\\
$$
w^{n+1}(x,y)=\frac{\exp\large\left(-\frac{\int_{B_r}g(z)[u^n(x+z)-u^n(y+z)]^2dz }{h}\large\right)}{\int_{\widetilde{O}^c}\exp\large\left(-\frac{\int_{B_r}g(z)[u^n(x+z)-u^n(y+z)]^2dz
}{h}\large\right)dy}.
$$.

Step 2. Update Image:\\
$$
{u^{n+1}(x)=\frac{2\int_{\Omega}\int_{B_r}g(z)u^{n}(y+z)[\widetilde{w}^{n}(x-z,y)+\widetilde{w}^{n}(y,x-z)]dzdy+\widetilde{\lambda}v(x)}{2\int_{\Omega}\int_{B_r}g(z)[\widetilde{w}^{n}(x-z,y)+\widetilde{w}^{n}(y,x-z)]dzdy+\widetilde{\lambda}},}
$$
where $\widetilde{\lambda}=\Bigg\{\begin{array}{cc}
                                      \lambda,&x \in O^c\\
                                      0,&x \in O
                                    \end{array}
\Bigg.$.

Step 3. If the convergence condition $\frac{\|u^{n+1}-u^n\|^2}{\|u^n\|^2}$ less than a tolerant error, then stop. Otherwise, go to the Step 1.
\end{Alg}

Notice that this algorithm is a coupling problem of image denoising and inpainting. However, block based methods are greatly dependent on the size of patch, and different patch sizes are needed to deal with Gaussian noise and inpainting areas. For inpainting problem, on the one hand, the results with large patch sizes perform well in restoring the whole of structure, but fail in keeping details; on the other hand, if small patches are chosen, the computation process will take more time and the whole structure of images will be more difficult to reconstruct, especially for large missing regions. Therefore, the multiscale technique is used in \cite{Arias2011A} for image inpainting. Based on the above analysis, we naturally propose a two-stage decoupling process: inside the inpainting domain $O$, we can apply the patch NL-means inpainting scheme in \cite{Arias2011A}; while outside, we employ the BNL$H^1$ denoising model in \cite{Liu2017A}.

\section{Experimental results}\label{Experimentresult}

In this section, we numerically demonstrate the superior performance of our proposed algorithm on natural image restoration problems. Firstly, we compare the patch non-local means inpainting method in \cite{Arias2011A} with several existing inpainting methods: TV based PDE method, coherence transport method (CTM) and exemplar based method (EBM). Next, we further test the effectiveness of the proposed algorithm for image inpainting and denoising simultaneously.
All the experiments are run under Windows 7 and  MATLAB R2017a with Intel Core i5-5200U CPU@2.80GHz and 8GB memory.

\subsection{Inpainting}

In this section, we assume the test images are not polluted by noise and show the superiority of the patch non-local inpainting method on missing block completion. Fig.2(a) and Fig.3(a) display the clean original images: Columbia and Barbara. 
In order to see more details, we enlarge the inpainting region in the read rectangle and show it in the left lower corner of the image.
The inpainting regions are represented by white colour, see Fig.2(b) and Fig.3(b). We recover the contaminated images with different inpainting methods. The results of TV are shown in Fig.2(c) and Fig.3(c). We can observe that TV is not good at restoring texture information and large missing regions.
CTM can not recover the pillar on image Columbia in Fig.2(d) and gives somewhat strange results in Fig.3(d).
EBM has the ability to fill in textures, however, brings some artifacts in Fig.2(e) and Fig.3(e).
We can clearly see that the patch non-local inpainting method leads to much better results than all the other three methods both quantitatively and qualitatively.

\begin{figure}[h]
\begin{center}
  \includegraphics[width=1.0\linewidth]{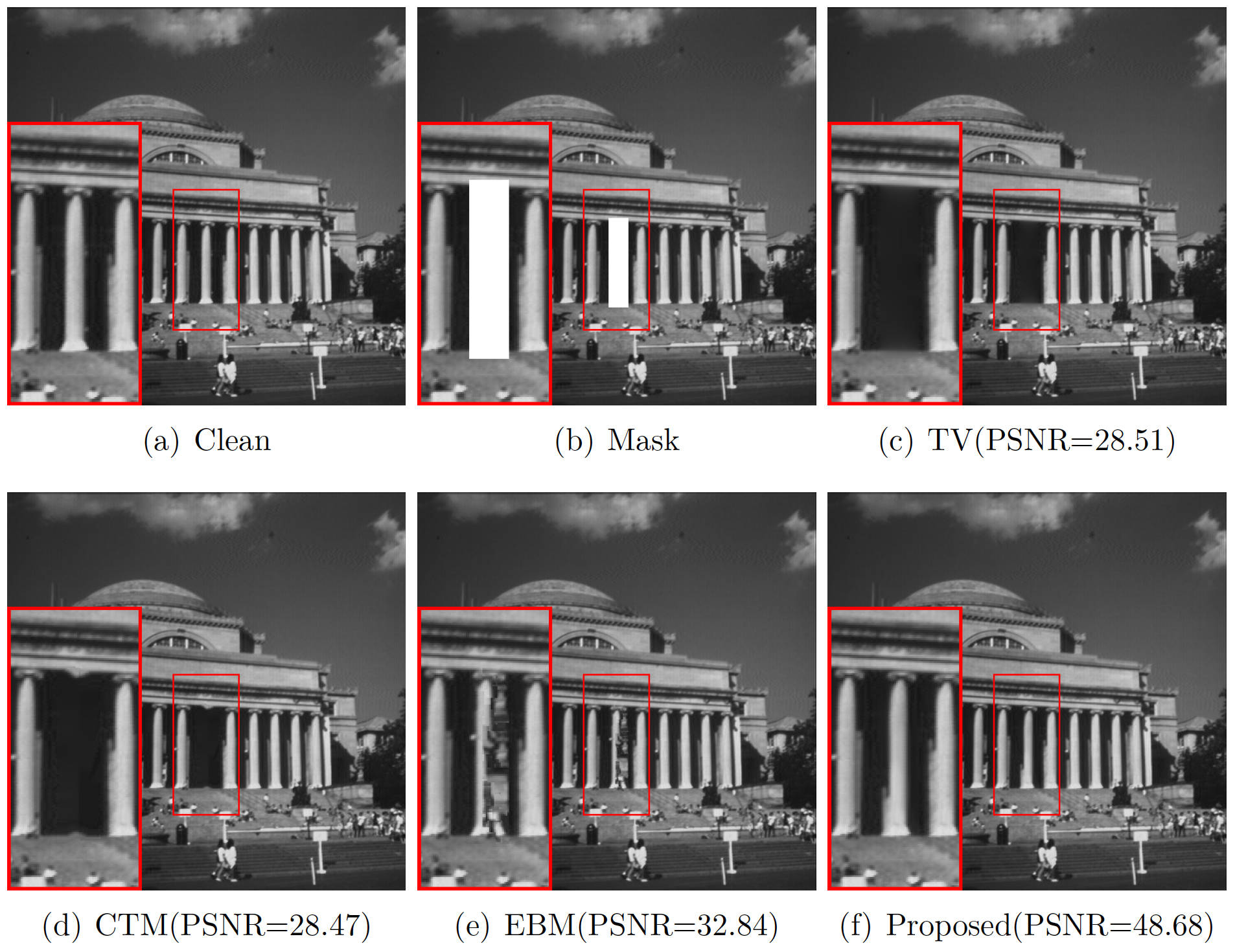}
%
\caption{Inpainting results of image Columbia by different inpainting methods.
The left lower images are the enlarged parts of the red rectangle regions.}
\label{figu3}
\end{center}
\end{figure}

\begin{figure}[h]
\begin{center}
  \includegraphics[width=1.0\linewidth]{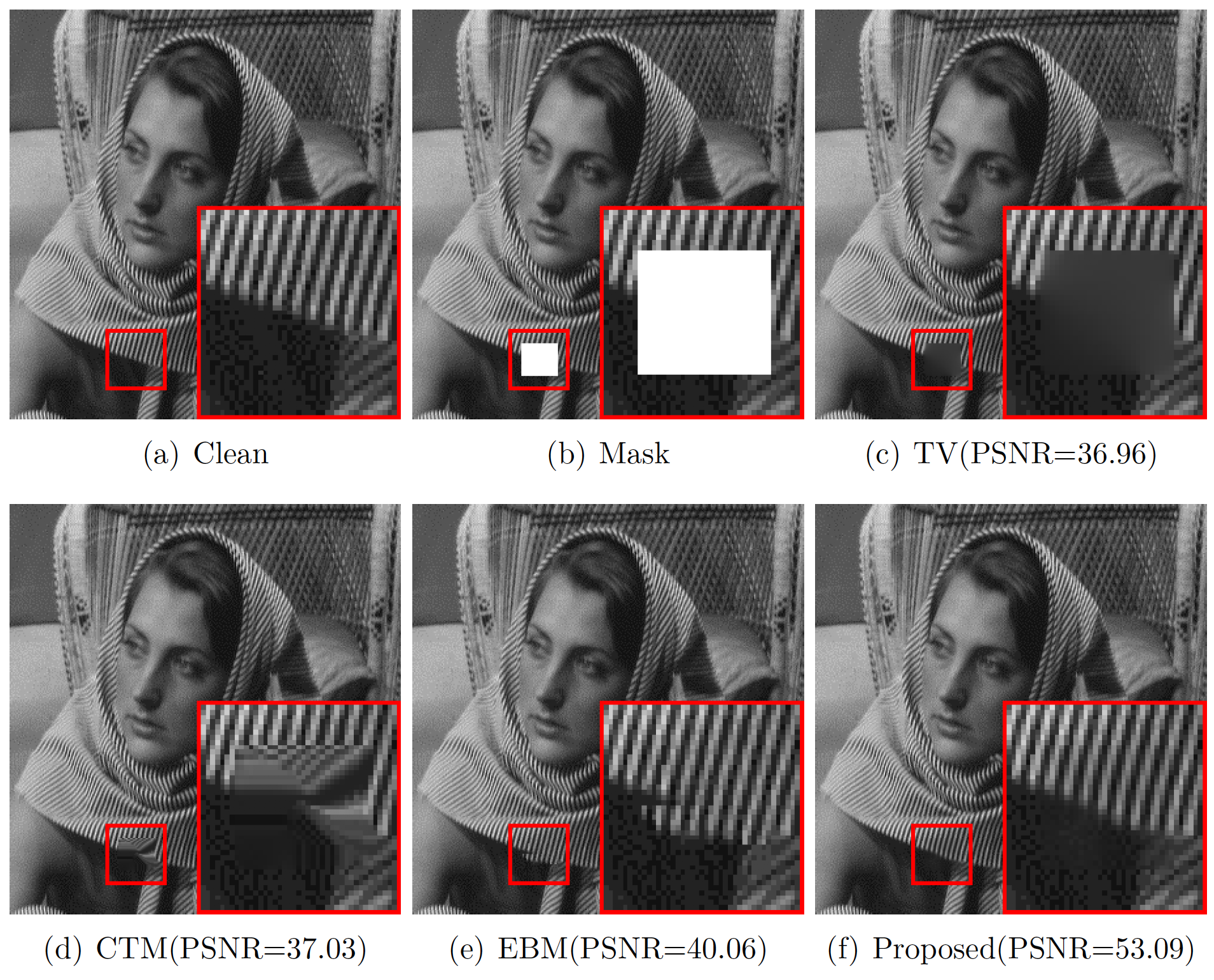}
%
\caption{Inpainting results of image Barbara by different inpainting methods. The right lower images are the enlarged parts of the red rectangle regions.}
\label{figu3}
\end{center}
\end{figure}

\subsection{Inpainting and Denoising}

In this section, we assume the inpainting images are contaminated by Gaussian white noise and apply the proposed algorithm on image inpainting and denoising simultaneously. The test images include Man, Bear and Boat.
The first column of \figurename s~\ref{figu2:1}~\ref{figu2:2}~\ref{figu2:3} are the original images and inpainting mask. The white color denotes the missing regions and the black one denotes the available regions.
Our purpose is to carry out image inpainting on the missing regions and image denoising on the available regions.
The second column shows the noisy images contaminated by Gaussian white noise with standard deviation $\sigma=10$ (top) and $\sigma=30$ (bottom).
In the third column, we display their corresponding reconstruction results of the patch NL-means method in \cite{Arias2011A}. We observe that these inpainting results still contain noise, because this method can not perform image denoising on the available regions.
The experimental results for the proposed algorithm are demonstrated in the last column.
It can be observed that the proposed method can realize image inpainting and denoising at the same time, and lead to visually perfect results. However, the higher the noise level, the more artifacts of the reconstructed images, such as the oversmoothing phenomenon on the soil of image Bear.

\begin{figure}[h]
\begin{center}
  \includegraphics[width=1.0\linewidth]{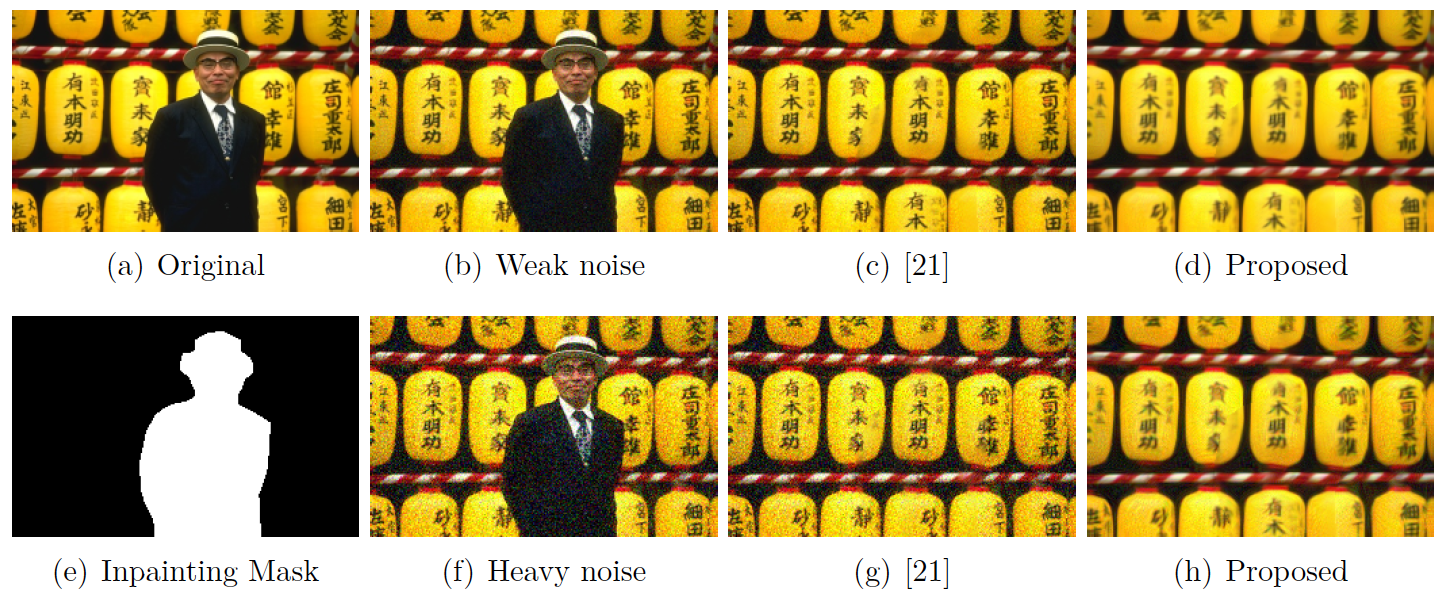}
\caption{Results of denoising and inpainting simultaneously. Left Column: original image and inpainting mask. Second column, from top to bottom: images contaminated by Gaussian noise with standard deviation $\sigma=10$ and $\sigma=30$. The third column shows their corresponding results of the patch NL-means method in \cite{Arias2011A}. The last column shows the simultaneous inpainting and denoising results of the proposed method.}
\label{figu2:1}
\end{center}
\end{figure}

\begin{figure}[h]
\begin{center}
  \includegraphics[width=1.0\linewidth]{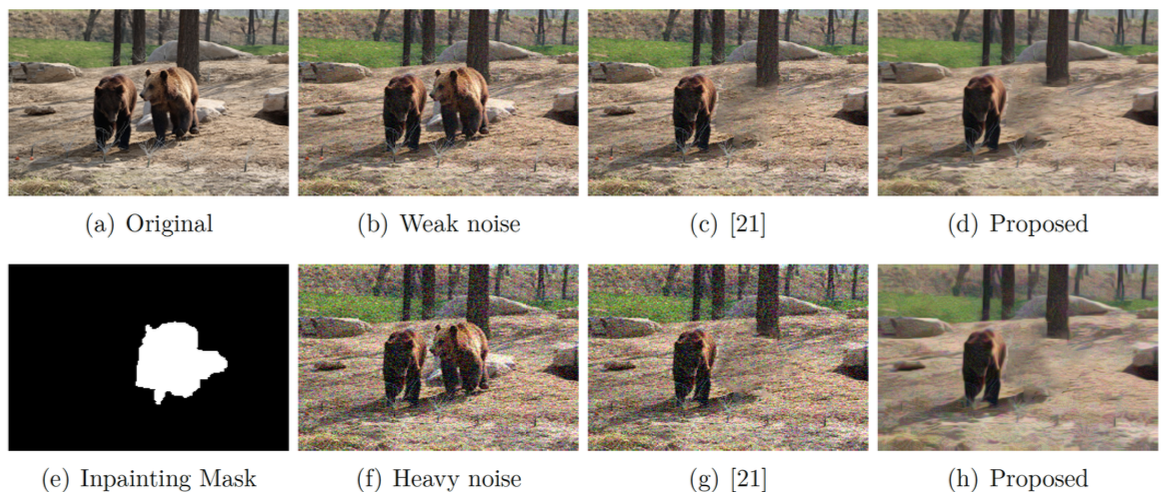}
\caption{Results of denoising and inpainting simultaneously. Left Column: original image and inpainting mask. Second column, from top to bottom: images contaminated by Gaussian noise with standard deviation $\sigma=10$ and $\sigma=30$. The third column shows their corresponding results of the patch NL-means method in \cite{Arias2011A}. The last column shows the simultaneous inpainting and denoising results of the proposed method.}
\label{figu2:2}
\end{center}
\end{figure}

\begin{figure}[h]
\begin{center}
  \includegraphics[width=1.0\linewidth]{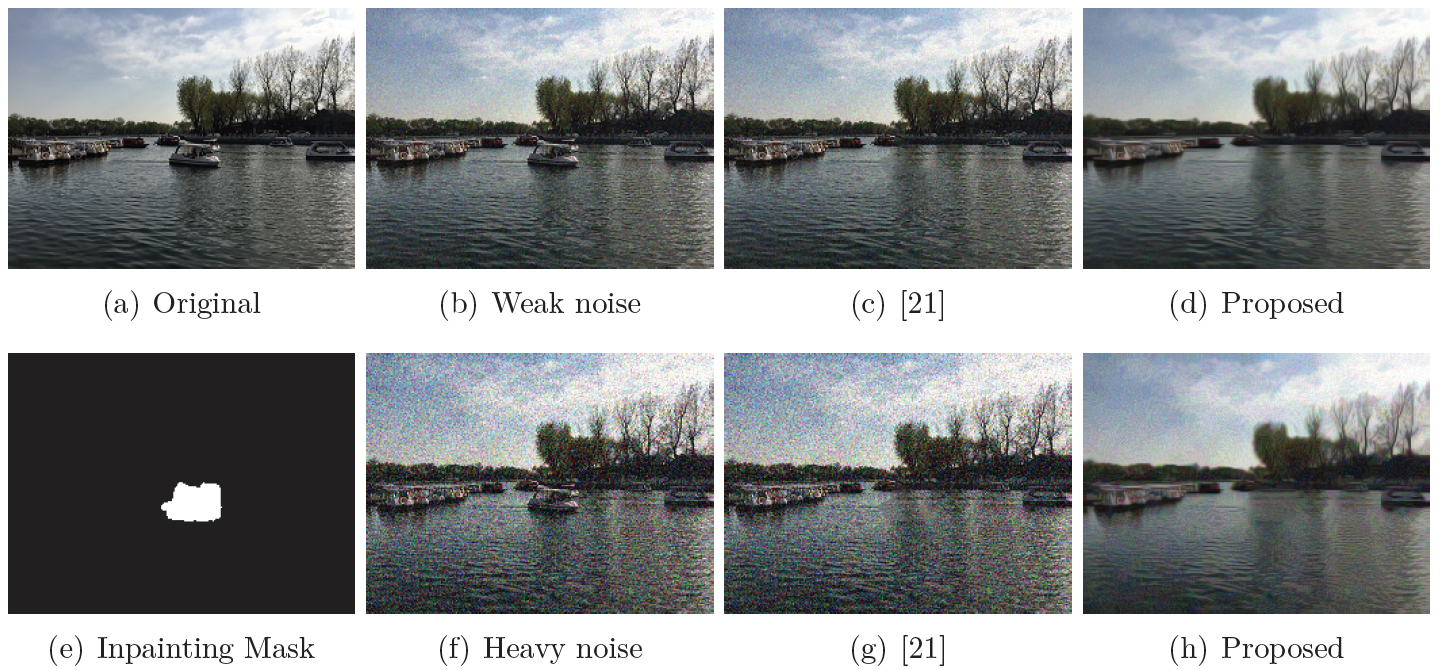}
\caption{Results of denoising and inpainting simultaneously. Left Column: original image and inpainting mask. Second column, from top to bottom: images contaminated by Gaussian noise with standard deviation $\sigma=10$ and $\sigma=30$. The third column shows their corresponding results of the patch NL-means method in \cite{Arias2011A}. The last column shows the simultaneous inpainting and denoising results of the proposed method.}
\label{figu2:3}
\end{center}
\end{figure}

\section{Conclusion}\label{ConclusionandFutureWork}

Motivated by the fact that the available regions of the original images are usually contaminated with noise, this work presents a general variational framework for block based non-local image inpainting. In the proposed model, image inpainting and denoising process are carried out simultaneously. Furthermore, we mathematically prove the existence of minimizer for the proposed model. To solve the proposed model efficiently, we present a decoupling algorithm. Experimental results show that it can provide some impressive results in image restoration.

\section*{Acknowledgement}
Jun Liu was supported by the National Natural Science Foundation of China (No. 11871035) and 
 the National Key Research and Development Program of China (2017YFA0604903).

\section*{Appendix}\label{sec:Appendix}
Under the hypothesis that $u_B(x),x \in \Omega$ are independent identically distributed with
$$
p(u_B(x);\Theta)=\sum\limits_{y \in \widetilde{O}^c}\frac{1}{|\widetilde{O}^c|}\delta_h(u_B(x)-u_B(y))=\sum\limits_{y \in \widetilde{O}^c}\frac{1}{|\widetilde{O}^c|}p(u_B(x);u_B(y),h),
$$
we have
$$
p(u_B;\Theta)=\prod\limits_{x \in \Omega}\sum\limits_{y \in \widetilde{O}^c}\frac{1}{|\widetilde{O}^c|}\delta_h(u_B(x)-u_B(y))=\prod\limits_{x \in \Omega}\sum\limits_{y \in \widetilde{O}^c}\frac{1}{|\widetilde{O}^c|}p(u_B(x);u_B(y),h),
$$
where $\Theta=\{h\}\bigcup\{u_B(y),y \in \widetilde{O}^c\}$ is an unknown set of pharameters.

Therefore, its negative log-likelihood function is as follows
$$
L(\Theta)=-\ln p(u_B;\Theta)=-\sum\limits_{x \in \Omega}\ln\sum\limits_{y \in \widetilde{O}^c}\frac{1}{|\widetilde{O}^c|}\delta_h(u_B(x)-u_B(y)).
$$

According to the full probability formula and the conditional probability formula, we obtain
\begin{align*}
L(\Theta)&=L(\Theta)\sum\limits_{\mathbf{c}}p(\mathbf{c}|u_B;\Theta^n)=-\sum\limits_{\mathbf{c}}p(\mathbf{c}|u_B;\Theta^n)\ln p(u_B;\Theta)\\
&=-\sum\limits_{\mathbf{c}}p(\mathbf{c}|u_B;\Theta^n)\ln \frac{p(u_B,\mathbf{c};\Theta)}{p(\mathbf{c}|u_B;\Theta)}\\
&=-\sum\limits_{\mathbf{c}}p(\mathbf{c}|u_B;\Theta^n)\ln p(u_B,\mathbf{c};\Theta)+\sum\limits_{\mathbf{c}}p(\mathbf{c}|u_B;\Theta^n)\ln p(\mathbf{c}|u_B;\Theta)\\
&:=Q(\Theta;\Theta^n)-H(\Theta;\Theta^n),
\end{align*}
where $\sum\limits_{\mathbf{c}}=\sum\limits_{c_1=1}^N\sum\limits_{c_2=1}^N\cdots\sum\limits_{c_N=1}^N$ and $N=|\widetilde{O}^c|$.\\

Under the i.i.d assumption of data, we have
\begin{align*}
p(\mathbf{c}|u_B;\Theta^n)&=\prod\limits_{y \in \widetilde{O}^c}p(c_j|u_B(y);\Theta^n),\\
p(u_B,\mathbf{c};\Theta)&=\prod\limits_{x \in \widetilde{O}^c}p(u_B(x),c_i;\Theta)=\prod\limits_{x \in \widetilde{O}^c}p(c_i;\Theta)p(u_B(x)|c_i;\Theta)=\prod\limits_{x \in \widetilde{O}^c}\alpha_{c_i}p(u_B(x)|c_i;\Theta).
\end{align*}
Then
\begin{align*}
Q(\Theta;\Theta^n)&=-\sum\limits_{\mathbf{c}}p(\mathbf{c}|u_B;\Theta^n)\ln p(u_B,\mathbf{c};\Theta)\\
&=-\sum\limits_{c_1=1}^N\sum\limits_{c_2=1}^N\cdots\sum\limits_{c_N=1}^N\sum\limits_{x \in \widetilde{O}^c} \ln (\alpha_{c_i}p(u_B(x)|c_i;\Theta))\prod\limits_{y \in \widetilde{O}^c}p(c_j|u_B(y);\Theta^n)\\
&=-\sum\limits_{c_1=1}^N\sum\limits_{c_2=1}^N\cdots\sum\limits_{c_N=1}^N\sum\limits_{x \in \widetilde{O}^c}\sum\limits_{y \in \widetilde{O}^c}\delta_{y,c_i} \ln (\alpha_{y}p(u_B(x)|y;\Theta))\prod\limits_{y \in \widetilde{O}^c}p(c_j|u_B(y);\Theta^n)\\
&=-\sum\limits_{x \in \widetilde{O}^c}\sum\limits_{y \in \widetilde{O}^c}\ln (\alpha_{y}p(u_B(x)|y;\Theta))\underbrace{\sum\limits_{c_1=1}^N\sum\limits_{c_2=1}^N\cdots\sum\limits_{c_N=1}^N\delta_{y,c_i}\prod\limits_{y \in \widetilde{O}^c}p(c_j|u_B(y);\Theta^n)}_I.
\end{align*}

Notice that
\begin{align*}
I&=\sum\limits_{c_1=1}^N\cdots\sum\limits_{c_{i-1}=1}^N\sum\limits_{c_{i+1}=1}^N\sum\limits_{c_N=1}^N\sum\limits_{c_i=1}^N\delta_{y,c_i}\prod\limits_{y \in \widetilde{O}^c}p(c_j|u_B(y);\Theta^n)\\
&=\sum\limits_{c_1=1}^N\cdots\sum\limits_{c_{i-1}=1}^N\sum\limits_{c_{i+1}=1}^N\sum\limits_{c_N=1}^N\prod\limits_{y \in \widetilde{O}^c,y\neq x}p(c_j|u_B(y);\Theta^n)p(y|u_B(x);\Theta^n)\\
&=p(y|u_B(x);\Theta^n)\prod\limits_{y \in \widetilde{O}^c,y\neq x}(\sum\limits_{c_j=1}^{N}p(c_j|u_B(y);\Theta^n))\\
&=p(y|u_B(x);\Theta^n).
\end{align*}

According to the Bayesian formula, we have
\begin{align*}
p(y|u_B(x);\Theta^n)=\frac{p(u_B(x),y;\Theta^n)}{p(u_B(x);\Theta^n)}=\frac{p(y;\Theta^n)p(u_B(x)|y;\Theta^n)}{p(u_B(x);\Theta^n)}=\frac{\alpha_yp(u_B(x)|y;\Theta^n)}{\sum\limits_{y \in \widetilde{O}^c}\alpha_yp(u_B(x)|y;\Theta^n)}.
\end{align*}
Therefore,
$$
p(y|u_B(x);\Theta^n)=\frac{p(u_B(x)|y;\Theta^n)}{\sum\limits_{y \in \widetilde{O}^c}p(u_B(x)|y;\Theta^n)}.
$$
Using $p(u_B(x)|y;\Theta^n)=p(u_B(x);u_B(y)^n,h^n)$, we have
$$
w^n(x,y)=\frac{p(u_B(x)|y;\Theta^n)}{\sum\limits_{y \in \widetilde{O}^c}p(u_B(x)|y;\Theta^n)}=\frac{p(u_B(x);u_B(y)^n,h^n)}{\sum\limits_{y \in \widetilde{O}^c}p(u_B(x);u_B(y)^n,h^n)}=\frac{\delta_h(u_B^n(x)-u_B^n(y))}{\sum\limits_{y \in \widetilde{O}^c}\delta_h(u_B^n(x)-u_B^n(y))},
$$
and
\begin{align*}
Q(\Theta;\Theta^n)&=-\sum\limits_{x \in \widetilde{O}^c}\sum\limits_{y \in \widetilde{O}^c}\ln (\alpha_{y}p(u_B(x)|y;\Theta))w^n(x,y)\\
&=-\sum\limits_{x \in \widetilde{O}^c}\sum\limits_{y \in \widetilde{O}^c}\ln (\alpha_{y}p(u_B(x);u_B(y),h))w^n(x,y)\\
&=-\sum\limits_{x \in \widetilde{O}^c}\sum\limits_{y \in \widetilde{O}^c}\ln (\alpha_{y}\delta_h(u_B(x)-u_B(y)))w^n(x,y)\\
&=-\sum\limits_{x \in \widetilde{O}^c}\sum\limits_{y \in \widetilde{O}^c}\ln (\delta_h(u_B(x)-u_B(y)))w^n(x,y)-\sum\limits_{x \in \widetilde{O}^c}\sum\limits_{y \in \widetilde{O}^c}\ln\alpha_{y}w^n(x,y).
\end{align*}
Thus, the M-Step in EM algorithm is
\begin{equation*}
u^{n+1}=\arg\min_{u}\large\left\{-\sum_{x \in \Omega}\sum_{y \in \widetilde{O}^c}\ln(\delta_h(u_B^n(x)-u_B^n(y)))w^{n}(x,y)\large\right\}.
\end{equation*}



\bibliographystyle{elsarticle-num}
\bibliography{bibtex}

\begin{thebibliography}{10}

\bibitem{Bertalmio2005Image}
Bertalmio, Marcelo, Sapiro, Guillermo, Caselles, Vincent, Ballester, and
  Coloma.
\newblock Image inpainting.
\newblock 4(9):417--424, 2005.

\bibitem{Rudin1992Nonlinear}
Leonid~I. Rudin, Stanley Osher, and Emad Fatemi.
\newblock Nonlinear total variation based noise removal algorithms.
\newblock {\em Physica D Nonlinear Phenomena}, 60(1–4):259--268, 1992.

\bibitem{Chan2001Mathematical}
Tony~F. Chan and Jianhong Shen.
\newblock Mathematical models for local nontexture inpaintings.
\newblock {\em SIAM Journal on Applied Mathematics}, 62(3):1019--1043, 2001.

\bibitem{Chan2001Nontexture}
Tony~F. Chan and Jianhong Shen.
\newblock Nontexture inpainting by curvature-driven diffusions.
\newblock {\em Journal of Visual Communication \& Image Representation},
  12(4):436--449, 2001.

\bibitem{Chan2002Euler}
Tony~F Chan, Sung~Ha Kang, and Jianhong Shen.
\newblock Euler's elastica and curvature based inpainting.
\newblock {\em SIAM Journal on Applied Mathematics}, 63(2):564--592, 2002.

\bibitem{Esedoglu2002Digital}
Selim Esedoglu and Jianhong Shen.
\newblock Digital inpainting based on the mumford-shah-euler image model.
\newblock {\em European Journal of Applied Mathematics}, 13(4):353--370, 2002.

\bibitem{Bertozzi2007Inpainting}
A.~L Bertozzi, S~Esedoglu, and A~Gillette.
\newblock Inpainting of binary images using the cahn-hilliard equation.
\newblock {\em IEEE Transactions on Image Processing A Publication of the IEEE
  Signal Processing Society}, 16(1):285, 2007.

\bibitem{Burger2009Cahn}
Martin Burger, Lin He, and Carola~Bibiane Nlieb.
\newblock Cahn-hilliard inpainting and a generalization for grayvalue images.
\newblock 2(4):1129--1167, 2009.

\bibitem{Tai2007Image}
Xue~Cheng Tai, Stanley Osher, and Randi Holm.
\newblock {\em Image Inpainting Using a TV-Stokes Equation}.
\newblock Springer Berlin Heidelberg, 2007.

\bibitem{Jun2012A}
Liu Jun, Tai Xue-Cheng, Huang Haiyang, and Huan Zhongdan.
\newblock A weighted dictionary learning model for denoising images corrupted
  by mixed noise.
\newblock {\em IEEE Transactions on Image Processing}, 22(3):1108--1120, 2012.

\bibitem{Hui2013Robust}
Ji~Hui, Sibin Huang, Zuowei Shen, and Yuhong Xu.
\newblock Robust video restoration by joint sparse and low rank matrix
  approximation.
\newblock {\em Siam Journal on Imaging Sciences}, 4(4):1122--1142, 2013.

\bibitem{Katkovnik2010From}
Vladimir Katkovnik, Alessandro Foi, Karen Egiazarian, and Jaakko Astola.
\newblock From local kernel to nonlocal multiple-model image denoising.
\newblock {\em International Journal of Computer Vision}, 86(1):1, 2010.

\bibitem{Kawai2009Image}
Norihiko Kawai, Tomokazu Sato, and Naokazu Yokoya.
\newblock Image inpainting considering brightness change and spatial locality
  of textures and its evaluation.
\newblock In {\em Pacific Rim Symposium on Advances in Image and Video
  Technology}, pages 271--282, 2009.

\bibitem{Wexler2004Space}
Y.~Wexler, E.~Shechtman, and M.~Irani.
\newblock Space-time video completion.
\newblock In {\em Computer Vision and Pattern Recognition, 2004. CVPR 2004.
  Proceedings of the 2004 IEEE Computer Society Conference on}, pages I--120--
  I--127 Vol.1, 2004.

\bibitem{Aujol2008Exemplar}
Jean~Fran?Ois Aujol, Sa?D Ladjal, and Simon Masnou.
\newblock Exemplar-based inpainting from a variational point of view.
\newblock {\em Siam J.math.anal}, 42(3):1246--1285, 2008.

\bibitem{Cao2009Geometrically}
Frédéric Cao, Yann Gousseau, Simon Masnou, and Patrick Pérez.
\newblock Geometrically guided exemplar-based inpainting.
\newblock {\em Siam Journal on Imaging Sciences}, 4(4):1143--1179, 2009.

\bibitem{Demanet2003Image}
Laurent Demanet, Bing Song, and Tony Chan.
\newblock Image inpainting by correspondence maps: a deterministic approach.
\newblock {\em Variational Level Set Methods, Prod. Of Workshop in Int"l Conf.
  Image Proc.}, 1100, 2003.

\bibitem{Criminisi2004Image}
A~Criminisi, P~Perez, and K~Toyama.
\newblock Region filling and object removal by exemplar-based image inpainting.
\newblock 13(9):1200--1212, 2004.

\bibitem{Li2014A}
Fang Li and Tieyong Zeng.
\newblock A universal variational framework for sparsity-based image
  inpainting.
\newblock {\em IEEE Transactions on Image Processing}, 23(10):4242, 2014.

\bibitem{weiwan2018}
Wei Wan, Jun Liu, and Huang Haiyang.
\newblock Local block operators and tv regularization based image inpainting.
\newblock {\em Inverse Problems and Imaging}, 12(6):1389--1410, 2018.

\bibitem{Arias2011A}
Pablo Arias, Gabriele Facciolo, Vicent Caselles, and Guillermo Sapiro.
\newblock A variational framework for exemplar-based image inpainting.
\newblock {\em International Journal of Computer Vision}, 93(3):319--347, 2011.

\bibitem{Liu2017A}
Jun Liu and Xiaojun Zheng.
\newblock A block nonlocal tv method for image restoration.
\newblock {\em SIAM Journal on Imaging Sciences}, 10(2):920--941, 2017.

\end{thebibliography}





\end{document}